\documentclass[a4paper,11pt]{article}
\usepackage{amsmath,amssymb,amsthm}
\usepackage{amsfonts}
\usepackage{eucal}
\numberwithin{equation}{section}

\renewcommand{\a}{\alpha}
\renewcommand{\b}{\beta}

\renewcommand{\d}{\delta}
\newcommand{\e}{\varepsilon}
\newcommand{\f}{\varphi}
\newcommand{\g}{\psi}

\renewcommand{\l}{\lambda}

\newcommand{\s}{\sigma}
\newcommand{\x}{\xi}

\newcommand{\re}{\mathbb{R}}
\newcommand{\ze}{\mathbb{Z}}

\def\pa{\partial}

\newcommand{\fourier}{\mathcal{F}}

\newcommand{\bigpare}[1]{\bigl(#1\bigr)}
\newcommand{\biggpare}[1]{\biggl(#1\biggr)}

\newcommand{\biggbra}[1]{\biggl\{#1\biggr\}}

\newcommand{\norm}[1]{\| #1 \|}
\newcommand{\bignorm}[1]{\bigl\| #1 \bigr\|}

\newcommand{\jap}[1]{\langle #1 \rangle}

\newtheorem{thm}{Theorem}
\newtheorem{lem}[thm]{Lemma}

\theoremstyle{definition}

\theoremstyle{remark}
\newtheorem{rem}{Remark}[section]

\title{Mathematical analysis of one-layer neural network with fixed biases, 
a new activation function and other observations}

\author{Fabricio Maci\`a\footnote{%
M$^2$ASAI, Universidad Polit\'ecnica de Madrid. ETSI Navales. Avda. de la Memoria, 4. 28040 Madrid, Spain. 
E-mail address: {\tt fabricio.macia@upm.es}}
\and Shu Nakamura\footnote{%
Department of Mathematics, Faculty of Sciences, Gakushuin University,
1-5-1, Mejiro, Toshima, Tokyo, 
171-8588 Japan
E-mail address: {\tt shu.nakamura@gakushuin.ac.jp}
}
}

\begin{document}
\maketitle

\begin{abstract}
We analyze a simple one-hidden-layer neural network with ReLU activation functions and fixed biases, 
with one-dimensional input and output.
We study both continuous and discrete versions of the model, and we rigorously prove the convergence of 
the learning process with the $L^2$ squared loss function and the gradient descent procedure. 
We also prove the spectral bias property for this learning process.

Several conclusions of this analysis are discussed; in particular, regarding the structure and properties that activation functions should possess, as well as the relationships between the spectrum of certain operators and the learning process. 
Based on this, we also propose an alternative activation function, the full-wave rectified exponential function (FReX),
and we discuss the convergence of the gradient descent with this alternative activation function. 
\end{abstract}

%%%%%%%%%%%%%%%%%%%%%%%
\section{Introduction}

We consider simple one-hidden-layer ReLU neural networks with one dimensional input and one dimensional output, 
where the first layer is fixed; namely, the weight is the identity, and the biases are preset. 
We consider both continuous models and discrete models (see the definitions below). 
The continuous model is essentially equivalent to the usual (continuous) neural network, whereas the discrete
model is different but closely related to the usual neural network. 
We show that learning with the standard $L^2$ loss function and the gradient descent procedure 
converges to the unique minima (without random initial conditions). In particular, we rigorously prove the spectral bias 
property for these models. Also, this analysis suggests that the ReLU activation is effective 
(at least partly) because it is a fundamental solution to the one-dimensional Laplacian. We then propose a new activation 
function, FReX, which is also a fundamental solution to a second order differential operator, 
but it also decays exponentially at infinity; hence, in particular, it is integrable and almost localized. 
We show that we can prove the same results as above for the models with this new activation function. 

In our models, the output is linear with respect to the trained parameters, and hence we can easily expect 
the learning process to be well-behaved, whereas these models are equivalent to or similar to standard 
neural network models. In particular, they have total representability of the $L^2$ functions for the 
continuous model and a natural class of representable functions for the discrete model. 
In these simple models, we can rigorously prove and analyze the convergence of the learning and spectral bias. We can also analyze the role of activation. 
Thus, we hope our simple models may play a role in understanding the 
basic mechanisms of neural networks in general. 

We note that the higher-dimensional models can be constructed and analyzed, but they are considerably 
more complicated, and we address this problem in a forthcoming paper.

Our starting point is the standard two-layer neural network: 
\[
y = \sum_{j=0}^{N-1} W^{(2)}_j \mathrm{ReLU}(W^{(1)}_j x +b_j)+\b,\quad  x\in\re, 
\]
where $N\in \mathbb{N}$, $W^{(k)}_j$ ($k=1,2,j=0,\dots, N-1$) are the weights, and 
$b_j$ ($j=0,\dots, N-1$) and $\b$ are biases. 
Here, $x\in\re$ is the input, or the independent variable, 
and $y\in\re$ is the output, or the prediction. 
$\mathrm{ReLU}$ is the rectified linear unit function, and we will denote it by $Y$ for simplicity:
\[
Y(z):= \mathrm{ReLU}(z)= \max(0, z), \quad z\in\re. 
\]
The above neural network admits the natural continuous analogue:
\begin{equation}\label{eq-fullmodel}
y (x)=\int_{-\infty}^{+\infty} \g^{(2)}(z) Y(\g^{(1)}(z)x+b(z)) dz + \b, \quad x\in\re, 
\end{equation}
where the weights $W^{(1)}_j$ and $W^{(2)}_j$ are replaced by $\g^{(1)}(z)$ and $\g^{(2)}(z)$, respectively, 
and the biases $b_j$ are replaced by a bias function $b(z)$. 

Our first observation, detailed in Section \ref{s-modelred}, is that the mathematical analysis of \eqref{eq-fullmodel} can be reduced to that of a simpler model of the form
\begin{equation}\label{eq-SimpModel}
    f(x) = \int_{-\infty}^{+\infty} \g(z) Y(x-z) dz +\b , \quad x\in\re. 
\end{equation}
We may interpret this as a one-hidden-layer (continuous) neural network with prescribed first-layer biases. There is only 
one weight $\g(z)$ and one bias $\b$ (they correspond to the weight and the bias in the second layer of the original formulation).  
The bias in the first layer is now considered an independent variable itself; hence, the terminology \textit{fixed} biases (see Section~3). 

Our goal in this article is to exploit this reduction to a simpler model in order to obtain insight into relevant aspects of the two-layer model. More specifically, we are interested in clarifying the role of the activation function in the representability aspect, the convergence of the associated gradient descent iteration, as well as the spectral bias analysis of the model.

Our second main observation is that the ReLU activation function satisfies: 
\begin{equation}\label{eq-FundamentalSolution}
Y''(z)=\d(z), 
\end{equation}
\textit{i.e.}, $Y$ is a fundamental solution to the one dimensional Laplace operator. 
Here $\d(z)$ is the one dimensional Dirac delta function, and the derivative is computed in the 
distribution sense on $\re$. At least formally, differentiating twice \eqref{eq-SimpModel} yields 
\[
f''(x) = \int_{-\infty}^{+\infty} \g(z) Y''(x-z)dz = \int_{-\infty}^{+\infty} \g(z)\d(x-z)dz = \g(x)
\]
by virtue of \eqref{eq-FundamentalSolution}. 
Thus, for a given $C^2$-class function  $f(x)$, the function $\g(x)=f''(x)$ is the {\em unique}\/ solution (optimal weight) for the one-hidden-layer neural network problem.\footnote{These assertions can be converted in rigorous mathematical statements provided one interprets the integrals in the sense of distributions and assumes polynomial growth at infinity both for $f$ and $\psi$. In Section \ref{s-simplem} a rigorous analysis is performed when the real line is replaced by a finite interval.}
After introducing a suitable loss function, the learning process will converge to its unique minimum, as it is proved in Section \ref{s-simplem}, and a discrete version is examined in Section \ref{s-discrete}.  

This suggests two important consequences. On one hand, property \eqref{eq-FundamentalSolution} implies that any sufficiently smooth function is representable; this relies on the singular second derivative of 
$Y$ at the origin. On the other hand, the uniqueness of this representation implies that the network is {\em exactly}-parametrized, 
in contrast to the usual neural networks, which are almost always very much over-parametrized.  

These remarks, which are elaborated in Section~4, allow us to conjecture that activation functions with the property of being fundamental solutions to a second-order differential operator should give rise to networks which possess good representability and convergence properties. 

Based on this, we propose an alternative activation function: the full-wave rectified exponential function (FReX)
\[
\mathrm{FReX}(x)=e^{-|x|}, \quad x\in\re.
\]
It is the unique fundamental solution of the second order operator $\frac12\bigpare{-\frac{d^2}{dx^2}+1}$; in other words, it satisfies
\[
\frac12\biggpare{-\frac{d^2}{dx^2}+1}\mathrm{FReX}(x)
=\d(x)
\]
in the distributional sense. 
In Section~5, we discuss the convergence of the learning process with this alternative activation function, as well as a spectral bias analysis.

The literature on neural networks and deep learning is vast, and we do not attempt a comprehensive review here. For general background, we refer to the textbooks of Prince \cite{Prince2023} and Goodfellow, Bengio, and Courville \cite{Goodfellowetal}. For topics closer to the present paper, we also mention Sonoda and Murata on unbounded activation functions, including ReLU \cite{SonodaMurata2017}, Antil et al. on fixed bias configurations \cite{AntilBrownetal}, and Dubey, Singh, and Chaudhuri for a survey of activation functions in deep learning \cite{DubeySinghChaudhuri}. Our arguments rely on methods from functional analysis and partial differential equations; for background in these areas, we refer, for example, to \cite{ReedSimon,FollandPDE}.

\subsection*{Acknowledgments}Part of this work was developed while F.M. was visiting Gakushuin University in Spring 2025 and Winter 2026; he wishes to thank this institution for its support and warm hospitality. F.M.'s research is supported by grants PID2021-124195NB-C31 and PID2024-158664NB-C21 from Agencia Estatal de Investigación (Spain). 

%
%%%%%%%%%%%%%%%%%%%%%%%
\section{The continuous model}\label{s-simplem}

\subsection{One-hidden-layer reduction}\label{s-modelred}

Our key observation is that, from the point of view of mathematical analysis, the two-layer neural network 
\begin{equation}\label{eq-twos}
    y(x) =\int_{-\infty}^{+\infty} \g^{(2)}(z) Y(\g^{(1)}(z)x+b(z)) dz + \b, \quad x\in\re, 
\end{equation}
with ReLU activation function $Y$ can be reduced to a one-hidden-layer network with fixed first-layer weights and biases:
\begin{equation}\label{eq-ones}
f(x) = \int_{-\infty}^{+\infty} \g(z) Y(x-z) dz +\b , \quad x\in\re. 
\end{equation}
Any function representable by \eqref{eq-twos} can be rewritten in the form \eqref{eq-ones} after reparameterization. This is done in two steps.
\medskip

\noindent\textbf{1. }We first simplify  the model by setting $\g^{(1)}(z)=1$. This is justified by noting that if $\g^{(1)}(x)>0$ for all $x$, then, since $Y(x)$ is homogeneous of degree one for positive values of $x$,
    \begin{align*}
    \int_{-\infty}^{+\infty} \g^{(2)}(z)&Y(\g^{(1)}(z)x+b(z))dz\\ &=\int_{-\infty}^{+\infty} \g^{(2)}(z)Y(\g^{(1)}(z)(x+(b(z)/\g^{(1)}(z))))dz\\
    &= \int_{-\infty}^{+\infty}\g^{(2)}(z)\g^{(1)}(z)Y(x+(b(z)/\g^{(1)}(z)))dz.
    \end{align*}
    In particular, we may assume  $\g^{(1)}(z)=1$ by replacing $\g^{(2)}(z)$ with 
    \[
    \g(z):=\g^{(2)}(z)\g^{(1)}(z). 
    \]
    If $\g^{(1)}(x)<0$ for certain values of $x$, an additional linear term, which is included in the model we discuss in the next section, must be added.\medskip

\noindent\textbf{2. }We can assume $b(z)=-z$; first note that, by the previous step, the bias $b(z)$ can be replaced by $b(z)/\g^{(1)}(z)$. By reordering the indices in the discrete model, we may assume that
$\{b_j\}$ is monotonically decreasing in $j$. Then, after a suitable change of variables, one may rewrite the integral so that the bias takes the form $b(z)=-z$.

\subsection{Definition of the network}
Motivated by the previous observations, we next describe the construction of a very simple continuous neural network to approximate (or represent) 
a smooth function $f(x)$ on the interval $I=[0,1]$. In addition to the weighted integral \eqref{eq-SimpModel}, restricting $x$ to $I$ requires the addition of a linear function term $b+cx$ with $b,c\in \re$ to the model.  We consider 
\[
g(x) = \int_0^1\g(z) Y(x-z) dz + b + c x, \quad x\in [0,1].
\]
As we have seen in the introduction, we have 
\[
\g(x) = g''(x), \quad x\in (0,1), 
\]
and direct computation gives
\[
b= g(0), \quad c= g'(0).
\]
In fact, it is possible to show that
\[
g(x) = \int_0^1 g''(y)Y(x-y) dy +g(0) + g'(0)x, \quad x\in[0,1], 
\]
provided $g$ is $C^2$-class.

\subsection{The loss function}

Here we consider the standard mean squared error (MSE): 
\[
L(f,g):= \mathrm{Loss}(f,g) = \int_0^1 |f(x)-g(x)|^2 dx, \quad f, g \in L^2(I),
\]
where $L^2(I)$ denotes the Lebesgue space consisting of square integrable functions on $I=[0,1]$. $L(f,g)$ coincides with the square of the $L^2$ norm of the difference between $f$ and $g$:
\begin{equation*}
    \norm{f-g}_{L^2}^2=L(f,g).
\end{equation*}

We compute the functional derivative of $g$ and $L(f,g)$ with respect to $(\g,b,c)\in\mathbb{W} :=L^2(I)\oplus\re^2$. First, $g$ is a linear function of all its arguments. Therefore, for $x\in (0,1)$, we have 
\[
\d g(x) = \int_0^1 \d\g(z)Y(x-z)dz +\d b + \d c \, x, 
\]
and hence
\[
\frac{\d g(x)}{\d \g}(h)= \int_0^1 h(z)Y(x-z)dz, 
\quad \frac{\d g(x)}{\d b} = 1\cdot, 
\quad \frac{\d g(x)}{\d c} = x\cdot,
\]
where $h(z)$ is a direction in $L^2(I)$ of the functional derivative. 
Then we note that $L(f,g)$ is quadratic in $g$. This yields
\[
\d L=-2 \int_0^1 (f(x)-g(x))\, \d g(x) dx,
\]
and the functional derivative of $L(f,g)$ is characterized by:
\begin{align*}
&\frac{\d L}{\d \g}(f,g)(z) = -2\int_0^1 (f(x)-g(x))Y(x-z)dx, \\
&\frac{\d L}{\d b}(f,g)  = -2\int_0^1 (f(x)-g(x))dx, \\
&\frac{\d L}{\d c}(f,g)  =-2\int_0^1 (f(x)-g(x))x\, dx.
\end{align*}

%%%
\subsection{Learning by gradient descent}

The gradient descent iteration in this context is 
\begin{align*}
&\g_{n+1}(z) = \g_n(z) +2\e \int_0^1 (f(x)-g_n(x))Y(x-z)dx, \\
&b_{n+1} = b_n +2\e \int_0^1 (f(x)-g_n(x))dx, \\
&c_{n+1} = c_n +2\e\int_0^1 (f(x)-g_n(x))x\, dx, 
\end{align*}
and
\[
g_n(x) = \int_0^1\g_n(z) Y(x-z) dz + b_n + c_n x, \quad x\in [0,1].
\]
for $n=0,1,\dots$ for some choice of $\g_0\in L^2(I)$, $b_0,c_0\in\re$. 
We will choose the learning rate $\e>0$ to be sufficiently small. 

\begin{thm}\label{t-conv}
For $f, \g_0\in L^2(I)$, $L(f,g_n)=\norm{f-g_n}_{L^2}^2\to 0$ as $n\to\infty$. 
\end{thm}

\begin{proof}
At first, we introduce several notations. 
We denote the space in which our parameters $(\g,b,c)$ live by $\mathbb{W} =L^2(I)\oplus\re^2$,
and let $T:\mathbb{W}\longrightarrow L^2(I)$ be defined by 
\[
T\f(x) =\int_0^1 Y(x-z)\g(z)dz + b +c x, \quad x\in I,
\]
for $\f=(\g,b,c)\in \mathbb{W}$. The adjoint $T^*:L^2(I)\longrightarrow\mathbb{W}$ of $T$ is given by
\[
T^* g = \biggpare{\int_0^1 Y(x-\cdot)g(x)dx, \int_0^1 g(x)dx, \int_0^1 xg(x)dx}. 
\]
The gradient descent iteration can be written in the more compact form
\[
g_n= T\f_n, \quad \f_{n+1} = \f_n+2\e T^*(f-g_n), \quad n=0, 1,2,\dots, 
\]
where $\f_n=(\g_n,b_n.c_n)$, and hence
\begin{align*}
f-g_{n+1} &= f- T\f_{n+1} = f-T\f_n-2\e TT^*(f-g_n)\\
&= (1-2\e TT^*)(f-g_n). 
\end{align*}
This implies 
\[
f-g_n = (1-2\e TT^*)^n(f-g_0), \quad n=0,1,\dots.
\]

\begin{lem}
$TT^*>0$ on $L^2(I)$, and $T^*T>0$ on $\mathbb{W}$. 
\end{lem}

\begin{proof}
It is obvious that $TT^*\geq 0$, and hence it suffices to show $T^*g=0$ implies $g=0$. 
If $T^*g=0$, we have 
\[
\int_0^1 Y(x-z)g(x)dx =0, \quad z\in [0,1]. 
\]
We differentiate this equation twice to have $g(x)=0$ for $x\in (0,1)$, 
and hence $g=0$ as an element of $L^2(I)$. 

Similarly, it suffices to show $\mathrm{Ker\, }T=\{0\}$ to conclude $T^*T>0$. 
If $T\f=0$ for $\f=(\g,b,c)\in\mathbb{W}$, then $(T\f)''=\g=0$ on $(0,1)$. 
Thus $T\f=b+cx$, but $T\f=0$ clearly implies $b=c=0$, and hence $\f=0$. 
\end{proof}

We note $A=TT^*$ is a compact self-adjoint operator. Thus, by the Hilbert-Schmidt theorem, 
there are eigenvalues: $\{\l_j\}\subset(0,\infty)$ and a corresponding orthonormal basis: $\{u_j\}$ 
such that $Au_j =\l_j u_j$ and $\l_j\to 0$ as $j\to\infty$. 
We choose $0<\e<(2\norm{TT^*})^{-1}$, so that $0<2\e\l_j<1$ for all $j$. We denote the inner product in $L^2(I)$ as follows:
\begin{equation*}
    \jap{f,g}:=\int_0^1f(x)g(x)dx,\qquad f,g\in L^2(I).
\end{equation*}
All the remarks above result in 
\[
f-g_n = \sum_{j=0}^\infty (1-2\e\l_j)^n \jap{u_j,f-g_0}u_j,
\]
which implies that
\[
\norm{f-g_n}_{L^2}^2 = \sum_{j=0}^\infty (1-2\e\l_j)^{2n} |\jap{u_j,f-g_0}|^2,
\]
converges to 0 as $n\to\infty$ by the dominated convergence theorem since 
\[
\sum_{j=1}^\infty |\jap{u_j,f-g_0}|^2 = \norm{f-g_0}_{L^2}^2<\infty,
\] 
and $(1-2\e\l_j)^n\to 0$ as $n\to\infty$ for each $j$. This concludes the proof of Theorem \ref{t-conv}.
\end{proof}

Note that Theorem~\ref{t-conv} does \textit{not}\/ imply the convergence of $\f_n=(\g_n,b_n,c_n)$
as $n\to\infty$. In fact, if it exists, $\lim \g_n$ would have to equal $f''$, and it would not be a function in $L^2(I)$
unless $f\in H^2(I):=\{f\mid f,f',f''\in L^2(I)\}$, the Sobolev space of order 2. Therefore, the next result is natural. 

\begin{thm}
If $f\in H^2(I)$, then $\f_n$ converges in $\mathbb{W}$, i.e., $\g_n$ converges in 
$L^2(I)$ and $b_n$ and $c_n$ converge in $\re$ as $n\to\infty$. 
\end{thm}

\begin{proof}
Suppose $f\in H^2(I)$ and set $\f=(f''(x),f(0),f'(0))\in \mathbb{W}$ so that $f=T\f$. Then 
\[
\f_{n+1}-\f =\f_n-\f +2\e T^*(f-T\f_n) =(1-2\e T^*T)(\f_n-\f),
\]
and hence 
\[
\f_n-\f = (1-2\e T^* T)^n (\f_0-\f).
\]
As in the proof of Theorem~1, we can show $(1-2\e T^* T)^n$ converges strongly to 0 
as $n\to\infty$, since $\s(T^* T)\setminus\{0\}=\s(T T^*)\setminus\{0\}$
and 0 is not an eigenvalue of $T$ by Lemma~2. 
This implies that $\f_n$ converges to $\f$ in $\mathbb{W}$ as $n\to\infty$. 
\end{proof}

If $f$ has a higher order of regularity, we can show that $\f_n$ converges at a faster rate. 

\begin{thm}
If $f\in H^{2+4k}(I)$ and $\g_0\in H^{4k}(I)$ with $k\in \mathbb{N}$,  then $\norm{\f_n-\f}_{\mathbb{W}}\leq C_k n^{-k}$
for some $C_k>0$. 
\end{thm}

\begin{proof}
By the assumption, we can write $f=T(T^*T)^k\tilde\f$ and $\f_0=(T^*T)^k\tilde\f_0$
with $\tilde\f, \tilde\f_0\in\mathbb{W}$. Then 
\[
\f_n-\f = (1-2\e T^*T)^n(T^*T)^k (\tilde\f_0-\tilde\f).
\]
We claim that
\begin{equation}\label{eq-smoothness-estimate}
\bignorm{(1-2\e T^*T)^n(T^*T)^k}_{L^2\to L^2}\leq C n^{-k}, \quad n\in\mathbb{N}.
\end{equation}
Let $\l>0$ be an eigenvalue of $T^*T$, and we recall $\e$ is chosen so that $2\e\l<1$.
Now we note
\[
(1-2\e\l)^n\l^k \leq e^{-2\e\l n}\l^k
\]
and the maximum with respect to $\l$ of the right hand is attained at $\l=k/(2\e n)$. 
Thus we have 
\[
(1-2\e\l)^n\l^k \leq e^{-k}\biggpare{\frac{k}{2\e n}}^k =C_k n^{-k}
\]
with $C_k>0$. This implies \eqref{eq-smoothness-estimate} and hence the conclusion. 
\end{proof}

%%%%%%%%%%%%%%%%%%%%%%%%%%%%%%%%%%

\section{Discrete model}\label{s-discrete}

Here we construct an analogous one layer (discrete) neural network with fixed biases%
\footnote{Our use of ‘fixed bias’ differs from that in \cite{AntilBrownetal}, where the emphasis is on fixed ordering rather than preset bias values.}. 

\subsection{Definition of the one-layer network}
We again consider functions on the interval $I=[0,1]$. Let $N\in \mathbb{N}$, and 
we divide $I$ into intervals of length $1/N$: 
\[
0=t_0<t_1<t_2<\cdots < t_{N-1}<t_N =1, \quad\text{with }t_j := j/N, j=0,\dots, N.
\]
Let $\mathfrak{X}$ be the set of continuous functions on $I$ that are linear on $[t_{j-1},t_j]$ for each $j$. 
Each $f\in \mathfrak{X}$ is determined by the values of $f$ on $\{t_j\}_{j=0}^N$, 
so it can be identified with a vector  $(f(t_0), \dots, f(t_N))\in \re^{N+1}$. 
We construct a network that represents any function in the space $\mathfrak{X}$ by 
a set of weights $(w(t_1),\dots,w(t_{N-1}))\in\re^{N-1}$ with additional parameters 
$b$ and $c\in\re$. 

We set
\begin{equation}\label{eq-DiscreteNetwork}
g(x) := \frac1N\sum_{j=1}^{N-1}   
w(t_j)\, Y(x-t_j)  + b + c x, \quad x\in I.
\end{equation}
One can check that $g\in \mathfrak{X}$. 

As in our continuous model, the term $b+c x$ must be added because of the presence of boundaries. 
We note that the sum is taken over $j\in \{1,\dots, N-1\}$, i.e., the interior points, 
and hence the dimension of the space of parameters $(w(t_1),\dots, w(t_{N-1}), b,c)$ is 
$N+1$, i.e., the same dimension as the function space $\mathfrak{X}$. 

By construction, at $t_0=0$, the function $g$ satisfies
\[
g(t_0)=b, \quad g'(t_0):=N(g(t_1)-g(t_0)) = c,
\]
which determines explicitly $b$ and $c$ from $g(x)$. We then define the discrete Laplace operator acting on $f\in \mathfrak{X}$ by 
\[
\triangle_N f(t_j)= N^{2}(f(t_{j-1})-2f(t_j)+f(t_{j+1})), \quad j=1,\dots N-1.
\]
We note that 
 \[
 \triangle_N \mathrm{ReLU}(z)=\triangle_N Y(z) = \begin{cases} N \quad &\text{if }z=0, \\
 0 \quad &\text{if }z\neq 0, \end{cases}
 \]
where $z\in N^{-1}\ze$. 
If we suppose that $g$ is given by \eqref{eq-DiscreteNetwork}, 
then the above formula implies
 \[
 \triangle_N g(t_j) =  w(t_j), \quad j=1,\dots, N-1.
 \]
 In fact, it is easy to verify
 \[
 g(x) = \frac1N\sum_{j=1}^{N-1} \triangle_N g(t_j) Y(x-t_j) + g(t_0)+g'(t_0)x, \quad x\in [0,1]
 \]
for any $g\in\mathfrak{X}$. 
Thus, each function $f\in \mathfrak{X}$ is represented by the above 
network with a unique $(w(t_1),\dots, w(t_{N-1}), b,c)\in\re^{N+1}$. 

\subsection{Loss function and learning procedure}
We set the $\ell^2$-norm on $\mathfrak{X}$ by 
\[
\norm{f}^2_\mathfrak{X}:=\frac1N\sum_{j=0}^N |f(t_j)|^2, \quad f\in\mathfrak{X}. 
\]
We now explicitly write the gradient descent iteration corresponding to the standard mean square loss function: 
\[
L(f,g) = \frac1N \sum_{j=0}^N (f(t_j)-g(t_j))^2=\norm{f-g}^2, \quad f,g\in \mathfrak{X}.
\]
For $g$ as in \eqref{eq-DiscreteNetwork}, we compute 
\[
\frac{\pa g}{\pa w(t_j)}(t_k) =  Y(t_k-t_j), \quad j=1,\dots, N-1, k=0,\dots, N,
\]
and 
\[
\frac{\pa g}{\pa b}=1, \quad \frac{\pa g}{\pa c}(t_k) = t_k,
\quad k=0,\dots, N. 
\]
Hence, we have 
\begin{align*}
&\frac{\pa L(f,g)}{\pa w(t_j)} = -\frac{2}{N}\sum_{k=0}^N (f(t_k)-g(t_k))Y(t_k-t_j),\\
&\frac{\pa L(f,g)}{\pa b} =-\frac2N\sum_{k=0}^N (f(t_k)-g(t_k)), \\
&\frac{\pa L(f,g)}{\pa c} = -\frac2N \sum_{k=0}^N (f(t_k)-g(t_k))t_k
\end{align*}
Again, following the gradient descent procedure, we set
\begin{align*}
& w_{n+1}(t_\ell) =  w_n(t_\ell) +\frac{2\e}{N}\sum_{j=0}^N (f(t_j)-g_n(t_j))Y(t_j-t_\ell)  \\
&b_{n+1} = b_n +\frac{2\e}{N} \sum_{j=0}^N (f(t_j)- g_n(t_j))  \\
&c_{n+1} =  c_n + \frac{2\e}{N} \sum_{j=0}^N (f(t_j)-g_n(t_j))t_j 
\end{align*}
and 
\[
g_{n+1}(t_\ell) = \frac1N \sum_{j=1}^{N-1} w_{n+1}(t_j)Y(t_\ell-t_j) + b_{n+1} + c_{n+1} t_\ell, 
\quad \ell=0,\dots, N, 
\]
for $n=0,1,\dots$ with some $\{w_0(t_j)\}_{j=1}^{N-1}\in \re^{N-1}$ and $b_0,c_0\in\re$. 

%%%
\subsection{Convergence of the learning dynamics}
We denote
\[
\mathbb{W}_N:=\re^{N-1}\oplus\re^2:
\]
and, for  $\f=(\{w(t_j)\}, b,c)\in\mathbb{W}$, write 
\[
\norm{\f}^2_{\mathbb{W}_N}=\frac1N\sum_{j=1}^{N-1} |w(t_j)|^2+|b|^2+|c|^2.
\]
Let $T:\mathbb{W}_N\longrightarrow\mathfrak{X}$ be given by 
\[
T\f(t_\ell) = \frac1N \sum_{j=1}^{N-1} w(t_j)Y(t_\ell-t_j) + b + c t_\ell, 
\quad \ell=0,\dots, N,
\]
for $\f=(\{w(t_j)\},b,c)\in\mathbb{W}_N$. One can check that
\[
T^*g =\biggpare{\biggbra{\frac{1}{N}\sum_{\ell=0}^N Y(t_\ell -\cdot)g(t_\ell)}_{\ell=1}^{N-1}, 
\frac1N\sum_{j=0}^N g(t_j), \frac1N\sum_{j=0}^N g(t_j)t_j}.
\]
We note, for $\f_n =(\{w_n(t_j)\},b_n,c_n)$, 
\[
g_n = T\f_n, \quad \f_{n+1} = \f_n+2\e T^* (f-g_n), \quad n=0,1,\dots.
\]
As was the case in the continuous model, we have 
\begin{align}
f-g_n &= f- T\f_n = f- T\f_{n-1}-2\e TT^*(f-g_{n-1}) \notag \\
&= (1-2\e TT^*)(f-g_{n-1})=(1-2\e T T^*)^n (f-g_0). \label{eq-disc-itr1}
\end{align}

\begin{lem}
$TT^*>0$ and $T^*T>0$. 
\end{lem}

\begin{proof}
Since $T T^*\geq 0$ is obvious, it suffices to show $\mathrm{Ker}(T^*)=\{0\}$. 
Suppose $T^* g=0$. Then $\triangle_N(T^* g)_1(t_j)=N^{-1} g(t_j)=0$ for $j=1,\dots,N-1$. 
Then $(T^*g)_2=0$ implies $\sum_{j=0}^N g(t_j) =0$, and hence $g(t_0)+g(t_N)=0$. 
Also, $(T^*g)_3=0$ implies $\sum_{j=0}^N g(t_j)t_j=0$ and hence $g(t_N)=0$. 
Thus we have $g(t_0)=g(t_N)=0$, and we conclude $g=0$. 

Similarly, we can show $\mathrm{Ker}(T)=\{0\}$. 
Suppose $T\f=0$ with $\f=(w,b,c)\in\mathbb{W}_N$. 
Then $\triangle_N T\f(t_j)=w(t_j)=0$ for $j=1,\dots, N-1$. 
Also, $T\f(t_0)=T\f(0)=b=0$ and $(T\f)'(0)=c=0$, and thus $\f=0$. 
\end{proof}

\begin{thm}
For $f\in \mathfrak{X}$ and $\f_0\in\mathbb{W}_N$, we set $\f_n\in\mathbb{W}_N$ and $g_n\in \mathfrak{X}$ 
for $n=1,2,\dots$, as above. Then  there exists $\f\in\mathbb{W}_N$ such that 
$f=T\f$ and $\norm{f-g_n}_{\mathfrak{X}}\to 0$, $\norm{\f_n-\f}_{\mathbb{W}_N}\to 0$
as $n\to\infty$. 
\end{thm}

\begin{proof}
Now we recall $\mathfrak{X}$ is finite dimensional, and hence $TT^*>0$ implies 
$TT^*\geq \a>0$. Thus, if $\e>0$ is chosen so that $2\e TT^*<1$, then 
$\norm{1-2\e TT^*}_{\mathfrak{X}\to\mathfrak{X}}\leq 1-2\e\a$ and hence 
$\norm{(1-2\e TT^*)^n}_{\mathfrak{X}\to\mathfrak{X}}\leq (1-2\e\a)^n \to 0$ as $n\to\infty$. 
By \eqref{eq-disc-itr1}, we conclude $\norm{f-g_n}_{\mathfrak{X}}\to 0$ as $n\to\infty$. 

Now we set 
\[
\f=(\{w(t_j)\},b,c)=(\{\triangle_N f(t_j)\}, f(0), f'(0))\in\mathbb{W}_N
\]
so that $f=T\f$. As we did in the proof of Theorem~3, we have 
\[
\f_n-\f = (1-2\e T^* T)^n(\f_0-\f)
\]
for $n=0,1,\dots$. By Lemma~5, $T^*T>0$. Hence, as $n\to\infty$, we have 
$\norm{(1-2\e T^* T)^n}_{\mathbb{W}_N\to\mathbb{W}_N}\to 0$ as well, 
and then $\norm{\f_n-\f}_{\mathbb{W}_N}\to 0$. 
\end{proof}

%%%%%%%%%%%%%%%%%%%%%%

\section{Several Observations}

Even though the networks considered here have a fairly simple structure, several conclusions can be drawn from our analysis. 

\subsection{Fixed bias models}

In neural networks, nonlinear behavior arises from the interaction between the activation function and the bias term. In particular, the basic nonlinear component should be viewed as
\[
a(x-b),
\]
where $a$ denotes the activation function and $b$ the bias. Our one-hidden-layer model is built directly on this observation. Specifically, we show that in the continuous setting, any function on $\mathbb{R}$ can be represented as a linear combination of such shifted nonlinear units, with $\mathrm{ReLU}$ as the activation. This is exactly the form of a one-hidden-layer neural network.

Moreover, this representation is unique, so the model admits an exact parameterization. We also show that gradient descent converges for learning this model in the continuous case, although the convergence rate is not always exponential.

These results highlight the central role of bias terms, alongside the activation and weight parameters. In numerical experiments with simple fully connected networks, we observed that the learned biases appear to spread roughly uniformly over a relevant interval. Since dense networks are invariant to permutations of hidden units, this suggests that biases may be fixed in advance, for example by uniformly sampling or uniformly placing them over an interval, with limited loss in performance. Such a design might reduce the number of trainable parameters and potentially improve training stability.

\subsection{Learning process}

In our model, the learning dynamics are described by a simple linear iteration
\[
\varphi_n \longmapsto \varphi_{n+1}
= \varphi_n + 2\varepsilon T^*(f - T\varphi_n)
= \varphi_n - 2\varepsilon T^*T(\varphi_n - \varphi),
\]
where $f = T\varphi$. Thus, convergence is governed by the spectral properties of the operator $T^*T$, or more precisely, of
\[
S_\varepsilon := I - 2\varepsilon T^*T,
\]
acting on the parameter space $\mathbb{W}$.

Since $T$ is a compact operator with a prescribed integral kernel (or matrix representation in the discrete case), the spectrum of $S_\varepsilon$ (away from the accumulation point at $1$) is contained in $(0,1)$. This implies convergence of the learning process for sufficiently small values of $\varepsilon > 0$. 

For the $\mathrm{ReLU}$ activation function, the integral kernel of $T^*T$ can be computed explicitly (we compute the kernel of its adjoint $TT^*$ explicitly in the next section). However, its support is rather broad on $[0,1]\times[0,1]$, making it difficult to determine which features are most relevant to the learning dynamics. This also makes it harder to extend the analysis to stochastic gradient descent. To understand stochastic gradient descent in a manner as concrete as gradient descent, a more refined description of the operator $T^*T$ may be necessary.

\subsection{Spectral bias analysis}
It is natural to study our simple model from the viewpoint of {\em spectral bias}, or the {\em frequency principle}, namely the tendency of gradient-based training to learn low-frequency components before high-frequency ones. This phenomenon has been widely observed for neural networks and has been analyzed both in Fourier terms and through kernel/NTK eigenmodes; see, for example,
\cite{Rahaman2019,Cao2021,Xu2025}. The explicit nature of the objects involved allows us to perform a rather precise analysis. 

The operator $A=TT^*$ is the integral operator
\[
(Au)(x)=\int_0^1 K(x,y)u(y)\,dy,
\]
with a symmetric kernel $K(x,y)=1+xy+\int_0^1 Y(x-z)Y(y-z)\,dz$, or, more explicitly,
\[
K(x,y)=
\begin{cases}
1+xy+\dfrac{x^2(3y-x)}{6}, & 0\le x\le y\le 1,\\[1.2ex]
1+xy+\dfrac{y^2(3x-y)}{6}, & 0\le y\le x\le 1.
\end{cases}.
\]
For every $f\in L^2(I)$, the function
\[
w:=Af=TT^*f
\]
is the unique solution to the boundary value problem
\[
w''''(x)=f(x) \qquad x\text{ in }(0,1),
\]
subject to the boundary conditions
\[
w'''(0)+w(0)=0,\qquad
w''(0)-w'(0)=0,\qquad
w''(1)=0,\qquad
w'''(1)=0.
\]
The formula for the error after $n$ iterations, $e_n=f-g_n$, becomes:
\begin{equation*}
    e_n = \sum_{j=0}^\infty (1-2\e\l_j)^n \jap{u_j,e_0}u_j,
\end{equation*}
where $\lambda_j$ are the eigenvalues of $A$ (which tend to zero as $j\to\infty$ since $A$ is compact) and $u_j$ are the corresponding eigenfunctions. Since $A$ is the inverse of a fourth order operator, one can conclude that $\lambda_j\sim cj^{-4}$ for a suitable constant $c>0$. This indicates that, after $n$ iterations, frequencies up to index $j\approx(\e n)^{1/4}$ are resolved.

\subsection{Functions of the activation and an alternative activation function}

The complete representability of our network, that is, the surjectivity of $T$, depends crucially on the property
\[
\mathrm{ReLU}''(x)=\d(x),
\]
namely, that the $\mathrm{ReLU}$ function is a fundamental solution of the one-dimensional Laplacian.

Fundamental solutions of second-order differential operators seem to be particularly useful in this context. Fundamental solutions of first-order operators are discontinuous and therefore probably not suitable for gradient-descent-based learning. By contrast, fundamental solutions of operators of order higher than two may be more complicated without offering clear additional advantages. In general, fundamental solutions of second-order differential operators are continuous but not differentiable at $0$, as in the case of $\mathrm{ReLU}(x)$. This feature of $\mathrm{ReLU}$, namely that it is continuous but not differentiable at $0$, is sometimes regarded as a disadvantage, but in our framework, it is in fact essential. On the other hand, $\mathrm{ReLU}$ diverges as $x\to\infty$ and is, in particular, not integrable. This may be a drawback in some settings, and the broad support of $T^*T$ described previously is one consequence.

A potentially useful fundamental solution of a second-order differential operator is
\[
e^{-|x|},
\]
which, in analogy with the $\mathrm{ReLU}$ function, may be called the Full-wave Rectified eXponential (FReX) function. We write
\[
Z(x)=\mathrm{FReX}(x)=e^{-|x|}, \quad x\in\re.
\]
It is the unique fundamental solution of $\frac12\bigpare{-\frac{d^2}{dx^2}+1}$, that is,
\[
\frac12\biggpare{-\frac{d^2}{dx^2}+1}\mathrm{FReX}(x)
=\frac12(-Z''(x)+Z(x))
=\d(x)
\]
in the distributional sense. We define
\[
T\f(x)=Z*\f(x)=\int Z(x-y)\f(y)\,dy, \quad \f\in L^2(\re),
\]
and it is well known that $T$ is a continuous bijection from $L^2(\re)$ to $H^2(\re)$. In particular, for any $f\in H^2(\re)$, if we set
\[
\f=\frac12(-f''+f)\in L^2(\re),
\]
then $f=T\f$. The analysis in Section~2 can then be carried out with this choice of $Z$ and $T$, without the boundary correction terms $b+cx$. From a theoretical point of view, the model with FReX activation is simpler and easier to analyze than the model with ReLU activation (see Section~5).

We also note that the operator $T^*T$, which appears in the learning process, can be computed explicitly for the FReX activation. The integral kernel of $T^*T$ is given by
\[
(1+|x-y|)e^{-|x-y|},
\]
which decays exponentially in the off-diagonal directions. This suggests that the learning process is fairly localized, which may be helpful for the analysis of stochastic gradient descent. We have also tested this activation function in a simple standard two-layer neural network for MNIST classification by replacing the ReLU activations with FReX activations. The resulting performance is roughly comparable to that of the ReLU network and clearly better than that of the Sigmoid network. Further investigation is needed to assess the viability of this alternative activation function, but at least it appears to serve as a reasonable activation in practice.

We emphasize that FReX is quite different from previously studied activation functions such as Sigmoid, Tanh, ReLU, Leaky ReLU, SoftPlus, GeLU, ELU, SELU, and Swish. These functions typically model some form of switching behavior and are usually monotone, or nearly monotone, increasing functions. Smoother functions are often considered preferable (see, e.g., \cite{DubeySinghChaudhuri}). By contrast, FReX is an even function: it is increasing on the negative half-line and decreasing on the positive half-line; it is as singular as ReLU at zero; and it decays exponentially as $x\to\pm\infty$, which raises the possibility of vanishing derivatives. Despite these unconventional features, FReX appears to work well as an activation function, and its justification is supported by the above analysis of simple one-hidden layer networks.

%%%%%%%%%%%%%%%%%%%%
\section{Convergence of the gradient descent with FReX activation}

In this section, we discuss the convergence of the learning process of the one-layer neural network 
with the FReX activation function. Since $\mathrm{FReX}(x)$ decays exponentially as $|x|\to\infty$, 
no boundary conditions or additional terms are needed, as was the case when the
$\mathrm{ReLU}$ activation function was used. For this reason, our analysis is carried out on $L^2(\re)$ or $\ell^2(\ze)$, which simplifies technical aspects considerably. 
In addition, the role of the smoothness of the target function, which is related to the spectral bias, is more transparent in this setting. 

\subsection{The continuous model}
We denote
\[
H_0=\frac12\biggpare{-\frac{d^2}{dx^2} +1 },
\]
viewed as an unbounded operator on $L^2(\re)$, with domain $\mathcal{D}(H_0)=H^2(\re)$. As noted in Section~4, we have 
$H_0Z(x)=\d(x)$ in the sense of distribution, \textit{i.e.}, $Z(x)$ is a fundamental solution to $H_0$. This can be equivalently stated as $Z*\f= H_0^{-1}\f$ for $\f\in L^2(\re)$, or 
\[
\f(x)= \int_{-\infty}^\infty Z(x-y)(H_0\f)(y)dy, \quad x\in\re,
\]
where $\f\in H^2(\re)$. 

For the gradient descent, we use the same loss function $L(f,g)=\norm{f-g}^2_{L^2}$ as in Section~2, 
and hence the gradient descent procedure is also the same (but without boundary terms): 
\[
\g_{n+1}(x) =\g_n(x) +2\e\int_{-\infty}^\infty (f(x)-g_n(x))Z(x-z)dx, 
\]
where $\e>0$ is the learning rate; we set
\[
g_n(x) =\int_{-\infty}^\infty \g_n(z)Z(x-z) dz. 
\]
Then we can show the following results. 

\begin{thm}
Suppose $f,\g_0\in L^2(\re)$ and $0<\e<1/8$. Then $\norm{f-g_n}_{L^2}\to 0$ as $n\to\infty$. 
Moreover, if $f\in H^2(\re)$, then $\norm{\g_n-\g}_{L^2}\to 0$ as $n\to\infty$, 
where $\g=H_0 f$. 
\end{thm}

\begin{proof}
The proof is standard and straightforward. We set an operator $T$ on $L^2(\re)$ by 
\[
T\f(x) = \int_{-\infty}^\infty \f(z)Z(x-z) dz, \quad x\in\re, \f\in L^2(\re). 
\]
so that we can simplify the notation: $g_n = T \g_n$; $\g_{n+1} =\g_n +2\e T^*(f-g_n)$. 
Thus, as in the proof of Theorem~1, we have: 
\[
f-g_{n+1} = f-T\g_{n+1} = f-T\g_n -2\e TT^*(f-g_n)=(1-2\e T T^*)(f-g_n),
\]
and hence 
\begin{equation}\label{eq-FReXerr}
f-g_n = (1-2\e TT^*)^n(f-g_0), \quad n=0,1,\dots.     
\end{equation}

\begin{lem}
$\norm{T}_{L^2\to L^2}=2$ and $TT^*=T^*T>0$. 
\end{lem}

\begin{proof}
We denote the one-dimensional Fourier transform by
\[
\mathcal{F}\f(\x)=\int_{-\infty}^\infty e^{-2\pi i x\x} \f(x)dx, \quad \x\in\re,\quad \f\in\mathcal{S}(\re),
\]
where $\mathcal{S}(\re)$ stands for the Schwartz class, and the inverse Fourier transform by $\mathcal{F}^*$. The Fourier transform of $Z(x)$ is 
\[
\mathcal{F}Z(\x) = \frac{2}{1+(2\pi\x)^2}, \quad \x\in\re. 
\]
We also have 
\[
\mathcal{F} (T\f) (\xi)=\mathcal{F}(Z*\f)(\xi) =\mathcal{F}Z(\xi) \mathcal{F}\f(\xi).
\]
Thus $\mathcal{F}T\mathcal{F}^*$ is a multiplication 
operator by $\mathcal{F}Z$, the claims follow. 
\end{proof}

We may suppose $0<\e<1/8$ so that $0<1-2\e \mathcal{F}TT^*\mathcal{F}^*<1$  as functions. 
Noting that $TT^*$ is self-adjoint, we have 
$\norm{(1-2\e TT^*)^n(f-g_0)}_{L^2}\to 0$ as $n\to\infty$ for any $f$ and $\g_0$. Similarly, if $\g= H_0 f\in L^2(\re)$, we have 
\[
\g-\g_n = (1-2\e T^* T)^n(\g-\g_0)
\]
and hence $\norm{\g-\g_n}_{L^2}\to 0$ as $n\to\infty$ 
using the same argument as above. 
\end{proof}

\begin{rem}
Actually, in this case, we can write down the operator $(1-2\e TT^*)^n$
explicitly in Fourier space, see \eqref{eq-FReXerrprop} and \eqref{eq-FReXmult}.
One can conclude the above theorem almost directly from those formulas
(using the dominated convergence theorem). The analogue of 
Theorem~4 also follows easily from those expressions. 
\end{rem}
%%%%

\subsection{Spectral bias for the FReX model}

The analysis of FReX from the point of view of spectral bias is particularly simple. 
Denote $e_n:=f-g_n$; in view of \eqref{eq-FReXerr}, the Fourier space representation of the error is:
\begin{equation}\label{eq-FReXerrprop}
\mathcal{F}e_n(\xi)=r_\varepsilon(\xi)^n \mathcal{F}e_0(\xi),\qquad n=0,1,\hdots   
\end{equation}
where
\begin{equation}\label{eq-FReXmult}
r_\e(\xi):=
1-\frac{8\e}{(1+(2\pi \xi)^2)^2},
\end{equation}
is the explicit representation of the operator $(1-2\e TT^*)^n$
in Fourier space, \textit{i.e.} $\mathcal{F}(1-2\e TT^*)^n \mathcal{F}^*$ is just multiplication by $r_\e(\xi)=\bigpare{1-2\e\fourier Z(\x)^2}^n$.

This allows us to provide a quantitative description of the frequency dependence of the learning rate: let $0<\varepsilon<1/8$. Then, as mentioned before,
\[
0\le r_\e(\xi)<1, \qquad \xi\in\mathbb{R},
\]
and $r_\varepsilon(\xi)$ is strictly increasing in $|\xi|$. In particular,
the low-frequency Fourier modes of the error decay faster than the high-frequency modes.

In order to measure the significant number of iterations needed to resolve the error measured at frequency $\xi$ one can try to find out which are the $\xi$ for which the damping
becomes significant after $n$ steps, i.e.
\[
n(1-r_\e(\xi))\approx 1.
\]
Since
\[
1-r_\e(\xi)=\frac{8\e}{(1+(2\pi\xi)^2)^2},
\]
this gives
\[
\frac{8\e n}{(1+(2\pi\xi)^2)^2}\approx 1.
\]
Hence, the effective learned frequency range after $n$ iterations is given by
\[
|\xi| \lesssim 
\frac{(8\e n)^{1/4}}{2\pi}.
\]
In other words, as was the case with ReLU, the FReX model learns frequencies only up to order $n^{1/4}$
after $n$ gradient descent steps.
%%%%

\subsection{The discrete model}

Here, we consider the function space 
\[
\mathbb{W}_N=\ell^2(N^{-1}\ze)=\{\f\,:\,N^{-1}\ze\to\re\mid \norm{\f}_{\mathbb{W}_N}<\infty\},
\]
where 
\[
\norm{\f}_{\mathbb{W}_N}^2 =\frac1N\sum_{z\in N^{-1}\ze} |\f(z)|^2,
\]
and $N>0$ is a sufficiently large natural number. 

We use the same notation as in Section~3, and we define 
\[
T\f(z) =\frac{1}{N}\sum_{t\in N^{-1}\ze}  \f(t) Z(z-t), \quad z\in N^{-1}\ze,
\]
for $\f\in \mathbb{W}_N$. We note $T^*=T$. By straightforward computations, we have 
\[
\triangle_N Z(x)=-(Na_N+b_N)\d_{0,x} +b_N Z(x),\quad z\in N^{-1}\ze,
\]
where $\triangle_N$ is the discrete Laplacian on $N^{-1}\ze$, $\d_{0,x}$ is the Kronecker delta function, 
$a_N=2e^{-1/2N}(2N\sinh(1/2N))$ and $b_N= (2N\sinh(1/2N))^2$. 
Note $a_N\sim 2$ and $b_N\sim 1$ as $N\to\infty$. Hence, we have 
\[
H_0 Z(x) := c_N (-\triangle_N +b_N) Z(x) = N\d_{0,x}, \quad z\in N^{-1}\ze,
\]
where $c_N = (a_N+b_N/N)^{-1}$. Thus, we can think that $Z(x)$ is a fundamental solution to 
$H_0$, i.e., $H_0 T\f =\f$ for any $\f\in\mathbb{W}_N$, or equivalently, $T=H_0^{-1}$. 

Now, using the same loss function as in Section~3, i.e., $L(f,g)=\norm{f-g}_{\ell^2}^2$, 
we have the following gradient descent procedure: Let $\f_0\in \mathbb{W}_N$, and 
we define 
\[
g_n = T\f_n, \quad \f_{n+1} = \f_n + 2\e T^*(f-g_n), \quad n=0,1,\dots.
\]
We set $\f = H_0 f$; again, we have
\[
f-g_{n+1} = f- T\f_{n+1} =(1-2\e TT^*)(f-g_n)
\]
and hence 
\begin{align*}
f-g_n &= (1-2\e TT^*)^n(f-g_0), \\
\f-\f_n &= (1-2\e T^* T)^n (\f-\f_0)
\end{align*}
for $n=0,1,2,\dots$. 

\begin{lem}
$T$ is self-adjoint, and $0< \a_N\leq T \leq \b_N<\infty$. 
\end{lem}

\begin{proof}
Let $F_N$ be the discrete Fourier transform defined by 
\[
F_N f(\x) = \frac1N\sum_{z\in N^{-1}\ze} e^{-2\pi i z\xi} f(z), \quad \x\in \re/(N\ze), 
\]
and the inverse is given by its adjoint 
\[
F_N^*\f(z) = \int_0^N e^{2\pi i z\xi} \f(\x) d\x, \quad z\in N^{-1}\ze.
\]
By standard computation, we have
\begin{align*}
-\triangle_N F_N^*\f &= N^2 \int_0^N (-e^{2\pi i\x/N}-e^{-2\pi i\x/N} + 2)e^{2\pi iz\x}\f(\x)d\x \\
&= 4N^2 \int_0^N \sin^2(\pi\x/N)\f(\x)d\x
\end{align*}
and thus $F_N(-\triangle_N)F_N^*$ is a multiplication operator by $4N^2\sin^2(\pi\x/N)$, 
and hence it is positive and its operator norm is bounded by $4N^2$. This implies
\[
c_N  b_N\leq H_0 \leq c_N(4N^2+b_N),
\]
and hence
\[
0<\a_N=\frac{a_N+N^{-1}b_N}{4N^2+b_N} \leq T \leq \frac{a_N+N^{-1}b_N}{b_N}=\b_N <\infty. \qedhere
\]
\end{proof}

From this, it is possible to establish the convergence of the gradient descent iteration, as in the previous sections. 

\begin{thm}
Suppose $f,\f_0\in \mathbb{W}_N$ and $0<\e<1/2(b_N)^2$. Then $\norm{f-g_n}_{\mathbb{W}_N}\to 0$ as $n\to\infty$, 
and  $\norm{\f_n-\f}_{\mathbb{W}_N}\to 0$ as $n\to\infty$, where $\f=H_0 f$. 
\end{thm}

%%%%%%%%%%%%%%%%%%%%%%%
\section{Conclusion and discussion}

We have constructed very simple, essentially one-hidden-layer neural networks with fixed biases for representing one-variable functions in both the continuous and discrete settings. The weights are uniquely determined by differentiation, owing to the fact that the ReLU function is a fundamental solution of the one-dimensional Laplace operator. The learning process based on standard gradient descent with the mean-square loss works in both the continuous and discrete cases, and we prove that the iterates converge to the unique solution.

We conclude with several remarks:
\begin{itemize}
\item Although these networks are extremely simple and the solution is explicit, the basic neural-network strategy already works well and can be established rigorously.
\item Our analysis helps explain why ReLU is particularly effective as an activation function.
\item Biases and activation functions together generate a family of nonlinearities, and the biases may be preset rather than learned.
\item Although the learning process is natural once the solution is known explicitly, the proof of convergence is not entirely trivial.
\item In the continuous case, since the operator $T$ is compact, one cannot expect convergence in operator norm; only strong convergence, that is, convergence for each initial condition, can be proved. By contrast, in the discrete case, operator-norm convergence can be shown.
\item On the basis of our analysis of these simple models, we make several observations about biases and activation functions in more general neural networks. In particular, we discuss the need to analyze the operator $T^*T$ in the learning process.
\item We also discuss the properties required of an activation function, propose an alternative activation function, FReX, and prove convergence of the corresponding learning process.
\end{itemize}

This simple but precise analysis suggests many possible directions for future research. Extending the analysis to multidimensional inputs and outputs is a natural next step. Another natural direction is to study the relationship between continuous and discrete models---for example, how discrete models approximate continuous ones, and how convergence behaves as the width tends to infinity, with suitably scaled fixed biases. The analysis of deeper networks, possibly built from similarly simple layers, is another clear direction for theoretical investigation.

More practical applications of these ideas, as well as further numerical experiments, also provide promising directions for future work. At present, however, our primary interest is in the theoretical aspects.
%%%%%%%%%%%%%%%%%%%%%

\bibliographystyle{alpha}
\bibliography{DL_refs}
\end{document}